\documentclass[]{article}

\usepackage[]{iclr_conference}

\usepackage{geometry}
\usepackage{acronym}
\usepackage{amsmath, amsfonts}
\usepackage{amsthm}
\usepackage{booktabs}
\usepackage[hidelinks]{hyperref}
\usepackage[linesnumbered,ruled,vlined]{algorithm2e}
\usepackage{graphicx}
\usepackage{mathtools}
\usepackage{wrapfig}
\usepackage{xcolor}

\hypersetup{
	colorlinks=true,
	linkcolor=red,
	citecolor=green,
	filecolor=magenta,      
	urlcolor=blue
}

\newtheorem{lemma}{Lemma}[section]
\newtheorem*{remark}{Remark}
\newtheorem{theorem}{Theorem}

\DeclarePairedDelimiter\ceil{\lceil}{\rceil}

\iclrfinalcopy

\title{Efficient Generation of Binary Magic Squares}
\author{Alain Riou\\
	Sony Computer Science Laboratories -- Paris\\
	LTCI, Télécom Paris, Institut Polytechnique de Paris\\
	\texttt{alain.riou14000@yahoo.com}
}

\acrodef{BMS}{Binary Magic Square}

\newcommand{\m}{\{0, \dots, m-1\}}
\newcommand{\n}{\{0, \dots, n-1\}}

\begin{document}

\maketitle

\begin{abstract}
	We propose a simple algorithm for generating Binary Magic Squares (BMS), i.e., square binary matrices where the sum of all rows and all columns are equal.
	We show by induction that our algorithm always returns valid BMS with optimal theoretical complexity.
	We then extend our study to non-square Binary Magic Squares, formalize conditions on the sum of rows and columns for these BMS to exist, and show that a slight variant of our first algorithm can generate provably generate them.
	Finally, we publicly release two implementations of our algorithm as Python packages, including one that can generate several BMS in parallel using GPU acceleration.
\end{abstract}

\section{Introduction}

\begin{wrapfigure}{r}{0.3\textwidth}
	\centering
	\includegraphics[width=0.9\linewidth]{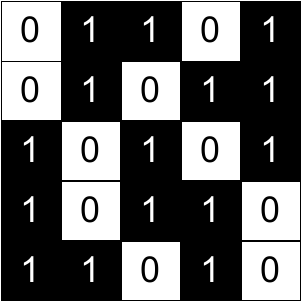}
	\caption{A Binary Magic Square with $n = 5$ and $k = 3$.}
	\label{fig:example}
	\vspace*{-8mm}
\end{wrapfigure}

Binary Magic Squares (BMS) are random binary matrices $M = (m_{ij})_{i,j} \in \{0, 1\}^{n \times n}$ such that the sum of all rows and columns is equal to the same constant $k$, i.e.
\begin{equation}
	\forall i \in \n, \ \sum_{j=0}^{n-1} m_{ij} = \sum_{j=0}^{n-1} m_{ji} = k.
\end{equation}
An example is provided on Figure~\ref{fig:example}.

BMS are strongly related to graph theory as they represent the adjacency matrix of regular bipartite graphs, leading to several applications, notably in combinatorial optimization~\cite{Chen1997}.


Several works have theoretically studied the existence and computability of binary matrices with constraints on the sum of rows and columns, among which BMS are a special case~\cite{Ryser1957,Nam1999,Alpers2017}.
A few works proposed to exhaustively generate all combinations of $k$-regular graphs~\cite{Meringer1999}. Markov chain Monte Carlo algorithms for sampling binary matrices with fixed margins, typically relying on column swapping, have also been proposed~\cite{Wang2019,Fout2020}.


Here, we propose an efficient algorithm for generating BMS. We demonstrate that it always returns valid BMS with a complexity of $\mathcal{O}(n^2)$, and we then generalize it to non-square BMS.
We release implementations of our algorithm as a Python pip-installable package.\footnote{\url{https://github.com/aRI0U/binary-magic-squares}}

\section{Preliminaries}

Before detailing our algorithm, we first observe that Binary Magic Squares always exist for any values of $n$ and $k$.

\begin{theorem}
	\label{th:existence}
	For all $n \in \mathbb{N}$, for all $k \in \{0, \dots, n\}$, there exists $M \in \{0, 1\}^{n \times n}$ such that $M$ is a valid BMS whose sum of rows and columns equals $k$.
\end{theorem} 
\begin{proof}
	Let $n \in \mathbb{N}$ and $k \in \{0, \dots, n\}$. Consider the matrix $M$ whose coefficients $m_{ij}$ are defined by
	\begin{equation}
		m_{ij} =
		\begin{cases}
			1 & \text{if} \ i \leq j < i + k \ \text{or} \ i \leq j + n < i + k \\
			0 &\text{otherwise}
		\end{cases}.
	\end{equation}
	
	For all $i \in \n$,
	\begin{equation}
		\begin{aligned}
			\sum_{j=0}^{n-1} m_{ij}
			&= \sum_{j=i}^{\min(i+k-1, n-1)} m_{ij} + \sum_{j=0}^{i + k - n - 1} m_{ij} \\
			&= \min(i+k, n) - i + \max(i + k - n, 0) \\
			&= k.
		\end{aligned}
	\end{equation}
	
	Similarly, for all $j \in \n$, 
	\begin{equation}
		\begin{aligned}
			\sum_{i=0}^{n-1} m_{ij}
			&= \sum_{i=\max(j-k+1, 0)}^{j} m_{ij} + \sum_{i=j+n-k+1}^{n-1} m_{ij} \\
			&= \min(j, k) + \max(k - j, 0) \\
			&= k.
		\end{aligned}
	\end{equation}
	
	By definition, $M$ is therefore a Binary Magic Square.
\end{proof}

This result, which is a special case of the Gale-Ryser theorem~\cite{Gale1957,Ryser1957}, provides us the guarantee that our algorithm can always return a valid BMS whatever the input arguments.

\section{Efficient Generation of BMS}

\subsection{Intuition}

The idea is to compute the BMS column by column. To do so, we start from a matrix $M \in \{0, 1\}^{n \times n}$ full of zeros and then, we successively pick $k$ indices $(i_0, \dots, i_{k-1})$ per column $t$ and put a 1 in the corresponding cells $(m_{i_l,t})_{0 \leq l < k}$, while perserving the following constraints on the sum over the rows through time:
\begin{equation}
	\label{eq:constraint}
	\forall i \in \n, \forall t \in \n, \ k + t - n \leq \sum_{j=0}^{t} m_{ij} \leq k.
\end{equation}
If Equation \eqref{eq:constraint} holds, in particular for $t = n-1$ the sum of each line $i$ is $\sum_{j=0}^{n-1} m_{ij} = k$.
Moreover, since we pick exactly $k$ indices per column, the sum of each column is also $k$ by construction, so $M$ is actually a Binary Magic Square.

\begin{figure}
	\centering
	\includegraphics[width=0.55\textwidth]{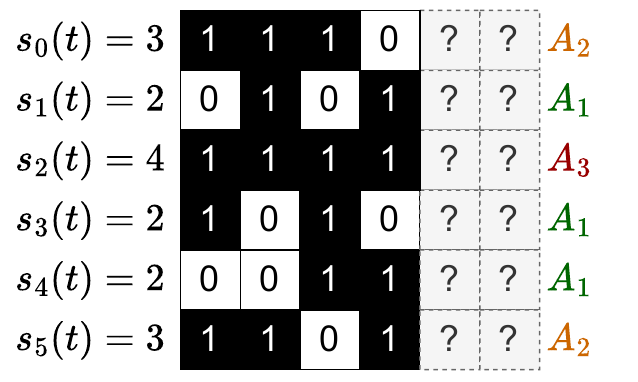}
	\label{fig:schema}
	\caption{Illustration of our algorithm for a BMS of size $6 \times 6$ with $k = 4$. At the end of step $t$, we compute the current sum of each row. Then, indices are partitioned based on their value. At step $t+1$, indices in $A_1$ must be selected to create a valid BMS, while indices from $A_3$ must not. Finally, some indices in $A_2$ are randomly selected (only one in this example) to ensure that the sum of each column equals $k$.}
\end{figure}

\pagebreak

\subsection{Algorithm}

	We now detail how to pick the indices at each step so that \eqref{eq:constraint} is satisfied at each time step.
	The idea, illustrated in Figure~\ref{fig:schema}, is that at each step $t$ we partition the candidate indices into three subsets $A_1$, $A_2$ and $A_3$ depending on whether the sum of the corresponding line is equal to $k + t - n$, equal to $k$ or strictly in between, and pick the right indices accordingly.

	\begin{algorithm}
		\label{algo:bms}
		\caption{Binary Magic Square generation}
		\SetKwInOut{Input}{input}
		\SetKwInOut{Output}{output}
		\Input{An integer $n \in \mathbb{N}$, an integer $k \in \{0, \dots, n\}$.}
		\Output{A Binary Magic Square $M \in \{0, 1\}^{n \times n}$ whose sum of rows and columns equals $k$.}
		$M = (m_{ij})_{i,j} = \mathbf{0}_{n \times n}$\;
		\For{$i \in \n$}{$s_i = 0$}
		\For{$t=0$ \KwTo $n-1$}{
			$A_1 := \{ i \in \n \,|\, s_i = k + t - n \}$\;
			$A_2 := \{ i \in \n \,|\, k + t - n < s_i < k \}$\;
			$A_3 := \{ i \in \n \,|\, s_i = k \}$\;
			$E := A_1 \cup \textbf{random\_subset}(A_2, k - |A_1|)$\;
			\For{$i \in E$}{
				$m_{it} := 1$\;
				$s_i := s_i + 1$\;
			}
		}
		\Return{$M$}
	\end{algorithm}

	Here, $E$ is the set of $k$ indices that is picked at each time step. We can easily prove that at the end of each iteration $t$, we have
	\begin{equation}
		s_i = \sum_{j=0}^t m_{ij}.
	\end{equation}
	
\section{Correction and complexity}

	\subsection{Correction}

	Here, we demonstrate that Algorithm~\ref{algo:bms} returns valid BMS, first by checking the sum of rows remains equal, then the sum of the columns.

	For each variable $x$ in Algorithm \autoref{algo:bms}, define $x(t)$ its $t$-th value in the algorithm.\footnote{Note that since $s_i$ are defined before the \textbf{for} loop the sequence $(s_i(t))_t$ is defined up to $t = n$ whereas other variables are defined only up to $t = n-1$.}
	In particular, the value of $s_i$ may increase only 1 by 1, i.e. for all $t$
	\begin{equation}
		\label{eq:one}
		s_i(t) \leq s_i(t+1) \leq s_i(t) + 1.
	\end{equation}
	
	\begin{lemma}
		\label{th:row}
		The sum of each row of a matrix generated by Algorithm~\ref{algo:bms} equals $k$.
	\end{lemma}

	\begin{proof}
	We show by induction that Equation \eqref{eq:constraint} is verified at each iteration $t$.

	\begin{itemize}
		\item $\forall i \in \n, s_i(0) = 0$ so equation \eqref{eq:constraint} holds for $t = 0$.
		
		\item Assume equation \eqref{eq:constraint} holds for one $t \in \{0, \dots, n-1 \}$. We show that it holds as well for $t+1$.
		
		By hypothesis, for all $i \in \n$, $k + t - n \leq s_i(t) \leq k$, so $(A_1(t), A_2(t), A_3(t))$ is a partition of $\n$.
		
		Then, for all $i \in \n$,
		\begin{itemize}
			\item $i \in A_1(t)$
			$$\begin{aligned}
				i \in A_1(t)
				&\Rightarrow s_i(t) = k + t - n \ \text{and} \ i \in E(t) \\
				&\Rightarrow s_i(t+1) = s_i(t) + 1 = k + t - n + 1
			\end{aligned}$$
			\item $i \in A_2(t)$
			$$\begin{aligned}
				i \in A_2(t)
				&\Rightarrow k + t - n + 1 \leq s_i(t) \leq k - 1 \\
				&\Rightarrow k + t - n + 1 \leq s_i(t + 1) \leq k &\text{by \eqref{eq:one}}
			\end{aligned}$$
			\item $i \in A_3(t)$
			$$\begin{aligned}
				i \in A_3(t)
				&\Rightarrow s_i(t) = k \ \text{and} \ i \notin E(t) \\
				&\Rightarrow s_i(t+1) = s_i(t) = k
			\end{aligned}$$
		\end{itemize}
		So for all $i \in \n$, $k + t - n + 1 \leq s_i(t+1) \leq k$.
		\item By induction,
		\begin{equation}
			\forall t \in \n, \forall i \in \n, k + t - n \leq s_i(t) \leq k.
		\end{equation}
	\end{itemize}
	
	In particular, for all $i \in \n$, $s_i(n) = k$, i.e. the sum of each row equals $k$.
	\end{proof}
	
	\begin{lemma}
		\label{th:col}
		The sum of each column of a matrix generated by Algorithm~\ref{algo:bms} equals $k$.
	\end{lemma}
	
	\begin{proof}
	For each row $i$ and column $t$, $m_{it} = 1$ if $i \in E(t)$ and 0 otherwise, by design. Therefore for all $t$
	\begin{equation}
		\sum_{i=0}^{n-1} m_{it} = |E(t)|.
	\end{equation}
	
	We show by induction that for all $t \in \n$, $|E(t)| = k$.
	\begin{itemize}
		\item Assume $t = 0$.
		\begin{itemize}
			\item If $k = 0$, then $A_1(0) = A_2(0) = \emptyset$ so $E(0) = \emptyset$ and $|E(0)| = 0 = k$.
			\item If $k = n$, then $k + t - n = 0$, i.e. $A_1(0) = \n$ and $A_2(0) = A_3(0) = \emptyset$.
			
			Then $E(0) = A_1(0) = \n$ and $|E(0)| = n = k$.
			\item If $0 < k < n$, then $A_1(0) = \emptyset$ and $A_2(0) = \n$.
			
			Then, $E(0)$ is a subset of $A_2(0)$ of size $k - A_1(0)$, i.e. a subset of $\n$ of size $k$. So $|E(0)| = k$.
		\end{itemize}
		
		\item Assume that exists $0 < t \leq n-1$ such that for all $0 \leq j < t$, $|E(j)| = k$. We show that $|E(t)| = k$ as well.
		
		At time step $t$, the total number of ones in $M$ is then $\sum_{j=0}^{t-1} |E(j)| = tk$ (i.e. the sum of the ones in each column).
		
		Also, the total number of ones in $M$ can be expressed as the sum of the number of ones in each line, i.e. $\sum_{i=0}^{n-1} s_i(t)$.
		
		As $A_1(t)$, $A_2(t)$ and $A_3(t)$ partition $\n$, we have therefore
		\begin{equation}
			\label{eq:linecol}
			\begin{aligned}
				tk
				&= \sum_{i \in A_1(t)} s_i(t) + \sum_{i \in A_2(t)} s_i(t) + \sum_{i \in A_3(t)} s_i(t) \\
				&= |A_1(t)|(k + t - n) + \sum_{i \in A_2(t)} s_i(t) + |A_3(t)| k
				&\text{by definition of } A_1 \text{ and } A_3
			\end{aligned}
		\end{equation}
	
		By definition of $A_2$,
		\begin{equation}
			\label{eq:flank}
			|A_2(t)| (k + t - n) < \sum_{i \in A_2(t)} s_i(t) < |A_2(t)| k.
		\end{equation}
	
		So, by injecting \eqref{eq:flank} into \eqref{eq:linecol},
		\begin{equation}
			(|A_1(t)| + |A_2(t)|)(k + t - n) + |A_3(t)| k < tk < |A_1(t)|(k + t - n) + (|A_2(t)| + |A_3(t)|) k.
		\end{equation}
	
		As $A_1(t)$, $A_2(t)$ and $A_3(t)$ partition $\n$,
		\begin{equation}
			|A_1(t)| + |A_2(t)| + |A_3(t)| = n.
		\end{equation}
		Therefore
		$$\begin{aligned}
			&(n - |A_3(t)|)(k + t - n) + |A_3(t)| k < tk \\
			&nk + nt - n^2 - |A_3(t)|k - |A_3(t)|t + |A_3(t)|n + |A_3(t)|k < tk \\
			&|A_3(t)|(n - t) < n^2 - nk - nt + tk \\
			&|A_3(t)|(n - t) < (n - k)(n - t) \\
			&|A_3(t)| < n - k
		\end{aligned}$$
		and
		$$\begin{aligned}
			&tk < |A_1(t)|(k + t - n) + (n - |A_1(t)|) k \\
			&tk < |A_1(t)|k + |A_1(t)|t - |A_1(t)|n + nk - |A_1(t)| k \\
			&|A_1(t)|(n - t) < nk - tk \\
			&|A_1(t)| < k
		\end{aligned}$$

		Then, $E(t) = A_1(t) \cup \textbf{random\_subset}(A_2(t), k - |A_1(t)|)$ by definition.
		$$\begin{aligned}
			|A_2(t)|
			&= n - |A_1(t)| - |A_3(t)| \\
			&> k - |A_1(t)| &\text{since }|A_3(t)| < n - k \\
			&> 0 &\text{since }|A_1(t)| < k
		\end{aligned}$$
		so $|\textbf{random\_subset}(A_2(t), k - |A_1(t)|)| = k - |A_1(t)|$ and $|E(t)| = k$ since $A_1(t) \cap A_2(t) = \emptyset$.
		
		\item By induction, for all $t \in \n$, $|E(t)| = k$.
	\end{itemize}
	\end{proof}
	
	\begin{theorem}
		Any matrix returned by Algorithm~\ref{algo:bms} is a Binary Magic Square.
	\end{theorem}
	
	\begin{proof}
		In Lemmas \ref{th:row} and \ref{th:col}, we proved that the sum of every row and column of a matrix $M$ generated by Algorithm \autoref{algo:bms} equals $k$. By definition, $M$ is therefore a Binary Magic Square.
	\end{proof}

\subsection{Complexity}

Without any parallelization trick, the overall complexity of Algorithm \autoref{algo:bms} is
\begin{equation}
	C(n) = \mathcal{O} \left( n^2 \right),
\end{equation}
which is the optimal theoretical complexity for writing a $n \times n$ matrix.

However, all operations inside the \textbf{for} loop can be done using vectorized operations in practice. If we have $p$ processes, the overall complexity of algorithm \autoref{algo:bms} then becomes
\begin{equation}
	C(n) = \mathcal{O} \left( n \ceil*{\dfrac{n}{p}}\right).
\end{equation}

\section{Towards non-square Binary Magic Squares}

	\subsection{Characterization}

	One can extend the definition of Binary Magic Squares to non-square matrices, by defining it as a matrix $M = (m_{ij})_{i,j} \in \{0, 1\}^{m \times n}$ such that
	\begin{equation}
		\label{eq:rbms}
		\begin{cases}
			\exists a \in \{0, \dots, n\}, \ \forall i \in \m, \quad \sum_{j=0}^{n-1} m_{ij} = a \\
			\exists b \in \{0, \dots, m\}, \ \forall j \in \n, \quad \sum_{i=0}^{m-1} m_{ij} = b \\
		\end{cases}.
	\end{equation}
	
	However, we show that not all combinations of $a, b, m, n$ can lead to valid magic squares.\\
	
	\begin{theorem}
		\label{th:nsq}
		Let $m, n \in \mathbb{N}$, $a \in \{0, \dots, n\}$ and $b \in \{0, \dots, n\}$. There exists a valid non-square BMS $M \in \{0, 1\}^{m \times n}$ s.t. the sum of every row (resp. column) equals $a$ (resp. $b$) iff the ratio between $b$ and $a$ equals the ratio between $m$ and $n$, i.e., exists $q, q', m', n' \in \mathbb{N}$ such that $m = qm'$, $n = qn'$, $b = q'm'$ and $a = q'n'$.
	\end{theorem}
	
	\begin{proof}
		Let $M = (m_{ij})_{i,j} \in \{0, 1\}^{m \times n}$ be such a matrix. By definition,
		\begin{equation}
			\sum_{i=0}^{m-1} \sum_{j=0}^{n-1} m_{ij} = \sum_{j=0}^{n-1} \sum_{i=0}^{m-1} m_{ij},
		\end{equation}
		which leads to
		\begin{equation}
			\label{eq:gauss}
			am = bn.
		\end{equation}
		
		Let $q = \gcd(m, n)$. There exists $m', n' \in \mathbb{N}$ such that $m = qm'$ and $n = qn'$, with $m' \wedge n' = 1$.
		
		Unless $a = b = 0$, follows that $m' | bn'$ and $n' | am'$, hence $m' | b'$ and $n' | a$ by Euclid's lemma.
		
		In other words, exists $q_a, q_b \in \mathbb{N}$ such that $a = q_a n'$ and $b = q_b m'$.
		
		Equation \eqref{eq:gauss} can be rewritten $q_a n' q m' = q_b m' q n'$, which ultimately leads to $q_a = q_b = q'$. Given $m, n \in \mathbb{N}$, the set of pairs $(a, b)$ that can lead to valid BMS is therefore
		\begin{equation}
			\{(q'n', q'm') \ | \ q = \gcd(m, n), m = qm', n = qn', q' \in \{0, \dots, q \} \}.
		\end{equation}

		We now demonstrate that under these conditions on $a, b$, there always exist a valid BMS.
		
		Let $M = (m_{ij})_{i,j} \in \{0, 1\}^{q \times q}$ be a BMS whose sum of rows and columns equals $q'$ (there exists at least one according to Theorem~\ref{th:existence}).
		
		Then, we construct $M' \in \{0, 1\}^{m \times n}$ by repeating $M$ $m'$ times vertically and $n'$ times horizontally. The sum of each row $i$ of $M$ equals $n' \sum_j m_{ij} = n'q' = a$. Similarly, the sum of each column $j$ equals $m' \sum_i m_ij = m' q' = b$, which concludes the proof.
	\end{proof}
	
	\begin{remark}
		The condition on $\gcd(m, n)$ implies, in particular, that if $m \wedge n = 1$, then the only valid BMS are the trivial ones ($a = b = 0$ or $a = n$ and $b = m$).
	\end{remark}
		
	\subsection{Algorithm}
	
	The algorithm and proof of its correctness are very similar to those detailed above for square BMS. We include both here for completeness.
	
	\begin{algorithm}
		\label{algo:bms_nsq}
		\caption{Non-square Binary Magic Square generation}
		\SetKwInOut{Input}{input}
		\SetKwInOut{Output}{output}
		\Input{Two integers $m, n \in \mathbb{N}$, two integers $a \in \{0, \dots, n\}$ and $b \in \{0, \dots, m\}$ satisfying the conditions of Theorem~\ref{th:nsq}.}
		\Output{A Binary Magic Square $M \in \{0, 1\}^{m \times n}$ whose sum of rows and columns equals $k$.}
		$M = (m_{ij})_{i,j} = \mathbf{0}_{m \times n}$\;
		\For{$i \in \m$}{$s_i = 0$}
		\For{$t=0$ \KwTo $n-1$}{
			$A_1 := \{ i \in \m \,|\, s_i = a + t - n \}$\;
			$A_2 := \{ i \in \m \,|\, a + t - n < s_i < a \}$\;
			$A_3 := \{ i \in \m \,|\, s_i = a \}$\;
			$E := A_1 \cup \textbf{random\_subset}(A_2, b - |A_1|)$\;
			\For{$i \in E$}{
				$m_{it} := 1$\;
				$s_i := s_i + 1$\;
			}
			\Return{M}
		}
	\end{algorithm}

	\begin{lemma}
		\label{th:row_nsq}
		The sum of each row of a matrix generated by Algorithm~\ref{algo:bms_nsq} equals $a$.
	\end{lemma}

	\begin{proof}
		Similarly to the square case, here we show that for all $i \in \{0, \dots, m-1\}$, one has
		\begin{equation}
			\label{eq:constraint_nsq}
			a + t - n \leq s_i(t) \leq a.
		\end{equation}
		We show by induction that Equation \eqref{eq:constraint_nsq} is verified at each iteration $t$.

		\begin{itemize}
			\item $\forall i \in \m, s_i(0) = 0$ so equation \eqref{eq:constraint_nsq} holds for $t = 0$.
			
			\item Assume equation \eqref{eq:constraint} holds for one $0 \leq t < n - 1$. We show that it holds as well for $t+1$.
			
			By hypothesis, for all $i \in \m$, $a + t - n \leq s_i(t) \leq a$, so $(A_1(t), A_2(t), A_3(t))$ is a partition of $\m$.
			
			Then, for all $i \in \m$,
			\begin{itemize}
				\item $i \in A_1(t)$
				$$\begin{aligned}
					i \in A_1(t)
					&\Rightarrow s_i(t) = a + t - n \ \text{and} \ i \in E(t) \\
					&\Rightarrow s_i(t+1) = s_i(t) + 1 = a + t - n + 1
				\end{aligned}$$
				\item $i \in A_2(t)$
				$$\begin{aligned}
					i \in A_2(t)
					&\Rightarrow a + t - n + 1 \leq s_i(t) \leq a - 1 \\
					&\Rightarrow a + t - n + 1 \leq s_i(t + 1) \leq a &\text{by equation \eqref{eq:one}}
				\end{aligned}$$
				\item $i \in A_3(t)$
				$$\begin{aligned}
					i \in A_3(t)
					&\Rightarrow s_i(t) = a \ \text{and} \ i \notin E(t) \\
					&\Rightarrow s_i(t+1) = s_i(t) = a
				\end{aligned}$$
			\end{itemize}
			So for all $i \in \n$, $a + t - n + 1 \leq s_i(t+1) \leq a$.
			\item By induction,
			\begin{equation}
				\forall t \in \n, \forall i \in \m, \ a + t - n \leq s_i(t) \leq a.
			\end{equation}
		\end{itemize}
		
		In particular, for all $i \in \m$, $s_i(n) = a$, i.e. the sum of each line equals $a$.
	\end{proof}

	\begin{lemma}
		\label{th:col_nsq}
		The sum of each column of a matrix generated by Algorithm~\ref{algo:bms_nsq} equals $b$.
	\end{lemma}
	
	\begin{proof}
		For each row $i$ and column $t$, $m_{it} = 1$ if $i \in E(t)$ and 0 otherwise, by design. Therefore, for all $t$,
		\begin{equation}
			\sum_{i=0}^{m-1} m_{it} = |E(t)|.
		\end{equation}
		
		We show by induction that for all $t \in \n$, $|E(t)| = b$.
		\begin{itemize}
			\item Assume $t = 0$.
			\begin{itemize}
				\item If $b = 0$, then $a = 0$.
				
				Consequently, $A_1(0) = A_2(0) = \emptyset$ so $E(0) = \emptyset$ and $|E(0)| = 0 = b$.
				
				\item If $b = m$, then $a = n$, and $a + t - n = 0$, i.e. $A_1(0) = \m$ and $A_2(0) = A_3(0) = \emptyset$.
				
				Then $E(0) = A_1(0) = \m$ and $|E(0)| = m = b$.
				\item If $0 < b < m$, then $0 < a < n$.
				
				Consequently, $A_1(0) = \emptyset$ and $A_2(0) = \m$.
				
				Then, $E(0)$ is a subset of $A_2(0)$ of size $b - |A_1(0)|$, i.e. a subset of $\m$ of size $b$. So $|E(0)| = b$.
			\end{itemize}
			
			\item Assume that exists $0 < t \leq n-1$ such that for all $0 \leq j < t$, $|E(j)| = b$. We show that $|E(t)| = b$ as well.
			
			At time step $t$, the total number of ones in $M$ is then $\sum_{j=0}^{t-1} |E(j)| = tb$ (i.e. the sum of the ones in each column).
			
			Also, the total number of ones in $M$ can be expressed as the sum of the number of ones in each line, i.e. $\sum_{i=0}^{n-1} s_i(t)$.
			
			As $A_1(t)$, $A_2(t)$ and $A_3(t)$ partition $\m$,  we have therefore
			\begin{equation}
				\label{eq:linecol2}
				\begin{aligned}
					tb
					&= \sum_{i \in A_1(t)} s_i(t) + \sum_{i \in A_2(t)} s_i(t) + \sum_{i \in A_3(t)} s_i(t) \\
					&= |A_1(t)|(a + t - n) + \sum_{i \in A_2(t)} s_i(t) + |A_3(t)| a
					&\text{by definition of } A_1 \text{ and } A_3
				\end{aligned}
			\end{equation}
			
			By definition of $A_2$,
			\begin{equation}
				\label{eq:flank2}
				|A_2(t)| (a + t - n) < \sum_{i \in A_2(t)} s_i(t) < |A_2(t)| a.
			\end{equation}
			
			So, by injecting \eqref{eq:flank2} into \eqref{eq:linecol2},
			\begin{equation}
				(|A_1(t)| + |A_2(t)|)(a + t - n) + |A_3(t)| a < tb < |A_1(t)|(a + t - n) + (|A_2(t)| + |A_3(t)|) a
			\end{equation}
			
			As $A_1(t)$, $A_2(t)$ and $A_3(t)$ partition $\m$,
			\begin{equation}
				|A_1(t)| + |A_2(t)| + |A_3(t)| = m.
			\end{equation}
			Therefore
			$$\begin{aligned}
				&(m - |A_3(t)|)(a + t - n) + |A_3(t)| a < tb \\
				&am + mt - mn - |A_3(t)|a - |A_3(t)|t + |A_3(t)|n + |A_3(t)|a < tb \\
				&|A_3(t)|(n - t) < mn - am - mt + tb \\
				&|A_3(t)|(n - t) < mn - bn - mt + tb &\text{ by \eqref{eq:gauss}} \\
				&|A_3(t)|(n - t) < (m - b)(n - t) \\
				&|A_3(t)| < m - b
			\end{aligned}$$
			and
			$$\begin{aligned}
				&tb < |A_1(t)|(a + t - n) + (m - |A_1(t)|) a \\
				&tb < |A_1(t)|a + |A_1(t)|t - |A_1(t)|n + am - |A_1(t)| a \\
				&|A_1(t)|(n - t) < am - tb \\
				&|A_1(t)|(n - t) < bn - tb &\text{ by \eqref{eq:gauss}} \\
				&|A_1(t)| < b
			\end{aligned}$$
			
			Then, $E(t) = A_1(t) \cup \textbf{random\_subset}(A_2(t), b - |A_1(t)|)$ by definition.
			$$\begin{aligned}
				|A_2(t)|
				&= m - |A_1(t)| - |A_3(t)| \\
				&> b - |A_1(t)| &\text{since }|A_3(t)| < m - b \\
				&> 0 &\text{since }|A_1(t)| < b
			\end{aligned}$$
			so $|\textbf{random\_subset}(A_2(t), b - |A_1(t)|)| = b - |A_1(t)|$ and $|E(t)| = b$ since $A_1(t) \cap A_2(t) = \emptyset$.
			
			\item By induction, for all $t \in \n$, $|E(t)| = b$.
		\end{itemize}
	\end{proof}

	\begin{theorem}
		Any matrix returned by Algorithm~\ref{algo:bms_nsq} is a non-square Binary Magic Square.
	\end{theorem}
	
	\begin{proof}
		In Lemmas \ref{th:row_nsq} and \ref{th:col_nsq}, we proved that the sum of every row (resp. column) of a matrix $M$ generated by Algorithm \ref{algo:bms_nsq} equals $a$ (resp. $b$). By definition, $M$ is therefore a Binary Magic Square.
	\end{proof}

\section{Conclusion}

We introduced an efficient algorithm and implementation for generating random Binary Magic Squares, and showed that it can easily be generalized to non-square matrices. Binary matrices with constraints on the sums of rows and columns have been widely studied, however simple algorithms with open implementations targeting the specific case of Binary Magic Squares were, to the best of our knowledge, lacking. Here, we therefore propose such an algorithm, theoretically validate its correctness, and publicly release a Python implementation. In particular, we also propose a PyTorch version of our algorithm which can generate many BMS in parallel on a GPU.

\bibliographystyle{ieeetr}
\bibliography{library}

\end{document}